\documentclass[numbib]{IMAIAI}
\setcitestyle{numbers} 

\usepackage{graphicx} 
\usepackage{subcaption} 
\usepackage{url} 
\usepackage{dsfont}
\allowdisplaybreaks
\usepackage{xcolor}
\usepackage{amssymb}
\usepackage{amsthm}
\usepackage{amsmath}

\begin{document}

\title{PACMAN: PAC-style bounds accounting for the Mismatch \\ between Accuracy and Negative log-loss}


\shorttitle{PACMAN}
\shortauthorlist{Vera, Rey Vega and Piantanida} 

\author{{
\sc Matias Vera}$^*$,\\[2pt]
CSC-CONICET and Universidad de Buenos Aires, Argentina\\
$^*${\email{mvera@fi.uba.ar}}\\[2pt]
{\sc Leonardo Rey Vega}\\[2pt]
CSC-CONICET and Universidad de Buenos Aires, Argentina\\
{lrey@fi.uba.ar}\\[6pt]
{\sc and}\\[6pt]
{\sc Pablo Piantanida} \\[2pt]
Laboratoire des Signaux et Systèmes (L2S)\\ Université Paris-Saclay CNRS CentraleSupélec, Gif-sur-Yvette, France\\
{pablo.piantanida@centralesupelec.fr}}

\maketitle

\begin{abstract}
{The ultimate performance of machine learning algorithms for classification tasks is  usually measured in terms of the empirical error probability (or accuracy) based on a testing dataset. Whereas, these algorithms are optimized through the minimization of a typically different--more convenient--loss function based on a training set. For classification tasks, this loss function is often the negative log-loss that leads to the well-known cross-entropy risk which is typically better behaved (from a numerical perspective) than the error probability. Conventional studies on the generalization error do not usually take into account the underlying mismatch between losses at training and testing phases. In this work, we introduce an analysis  based on point-wise PAC approach over the generalization gap considering the mismatch of testing based on the accuracy metric and training on the negative log-loss. We label this analysis PACMAN. Building on the fact that the mentioned mismatch can be written as a likelihood ratio,  concentration inequalities can be used to provide some insights for the generalization problem in terms of some point-wise PAC bounds depending  on some meaningful information-theoretic quantities. An analysis of the obtained bounds and a comparison with available results in the literature are also provided.}
{Generalization, Point-wise PAC-bounds, Log-loss, Accuracy, Likelihood Ratio.}
\\
2000 Math Subject Classification: 34K30, 35K57, 35Q80,  92D25
\end{abstract}

\section{Introduction}


Generalization abilities of an algorithm is a fundamentally important area in the machine learning~\cite{vapnik98statlearn}, \cite{Devroye97a}. The use of the empirical risk minimization (ERM) principle jointly with a well-chosen risk or loss function, requires that the learned models are able to work well for unseen examples dudring training, i.e. not present in the training dataset. In other words, we require the model is able to ``truly learn'' as opposed to ``simply memorize'' the training data. 

Several theoretical guarantees have been proposed to quantify how close the empirical risk is to the actual  risk. Most of them assume that the loss function is bounded,  as it is the case for the misclassification error. Some well-known results include: uniform convergence bounds \cite{Vapnik1971uniform}, PAC bounds \cite{Valiant:1984:TL:1968.1972} and PAC Bayesian bounds \cite{McAllester98somepac-bayesian}. In recent years it is more frequent to study deviations from the dataset and the hypothesis set simultaneously. These bounds consider that the training algorithm outputs a posterior distribution of the hypothesis given the training dataset. Some of these results consider averaged bounded risks over randomized classifiers with respect to the learned posterior distribution. Although typical PAC-bayesian bounds give valuable theoretical insights, some guarantees are needed for the actual risk of the trained model, and  not only for an average risk over randomized classifiers. It is for this reason that recently there was some interest in \emph{derandomized} bayesian bounds which are also known as point-wise PAC bounds~\cite{pmlr-v83-bassily18a,Esposito2020RobustGV,hetao,DBLP:journals/corr/abs-2102-08649}, among others. 

Most of the literature regarding the analysis of generalization gap does not distinguish between the loss function used to evaluate the true risk and that used at training time to perform ERM. Although both metrics are usually the same in regression problems, this does not happen in classification tasks. It is widely accepted that the negative log-loss function is a natural choice for the loss function in classification problems~\cite{Bishop_2007} where training is performed based on ERM. This metric allows to model the posterior distribution of the labels given the inputs and with an appropriate parametrization it leads to sufficiently well-behaved models that can be numerically optimized using techniques such as backpropagation \cite{Goodfellow-et-al-2016}. Despite of the benefits of this metric, most authors agree that the ultimate goal of any classification algorithm is to achieve a high accuracy or low classification error probability for examples not present in the training set. It is well known that negative log-loss and accuracy are not totally aligned, since not even their optimal solutions match \cite{guo17}.

However, most theoretical studies (as the ones cited above) focus on investigating generalization with respect to a only one metric, typically the classification error probability (or its complement, namely the accuracy). Whereas one of the main goals of most of the literature is to study how the models obtained through the minimization of the empirical error probability can achieve good generalization ability, the mismatch between what is typically done in practice (minimizing empirical negative log-loss at training) and what is the ultimate goal (minimizing classification error probability at testing) has not received sufficient attention. 

\subsection{Our Contributions}
The main contributions of this manuscript can be summarized as:
\begin{itemize}
    \item We present the PACMAN analysis: a novel point-wise PAC generalization study motivated by the above mismatch between the accuracy and the negative log-loss. Our analysis is different to the ones in the papers cited above. In first place, as we have already mentioned, the loss used at training is not the same that the one used at testing. Secondly, the negative log-loss is not bounded a function in general. Most of the results mentioned above depends on the use of bounded loss functions. Our analysis bypasses these issues observing that the considered generalization metric can be expressed as likelihood ratio between two distributions (one of them being the data distribution).
    \item Using simple and well-known concentration inequalities some non-trivial bounds of this likelihood ratio can be obtained. The obtained bounds shows the importance of information-theoretic metrics related with the R\'enyi $\alpha$-divergence \cite{erven14}, which is also the case for some of the more classical generalization results cited above that do not consider the mismatch between the loss at training and testing (Theorem \ref{thm:bayes} and Corollary \ref{cor:1}).
    \item We obtain a bound dominated by the mutual information between the training dataset and the chosen hypothesis, which is a known result for the classical generalization metric. Our PACMAN analysis supports the very well-known conclusion that learners which share little information generalizes well, but without restricting it to a $0-1$ \cite{pmlr-v83-bassily18a} or subgaussian loss \cite{Xu_Raginsky_2017} (Theorem \ref{thm:little} and Lemma \ref{lem:esperanzaesinformacion}).
    \item As the PACMAN generalization metric is a likelihood ratio, we also develop a novel bound based on a Neyman-Pearson decision problem (Theorem \ref{thm:ratio}). We are not aware of other similar result reported in the bibliography and for that reason we study this bound under three different assumptions: I. Finite hypothesis space (Ex. \ref{cor:finito}), II. Subgaussian hypothesis on the generalization metric (Ex. \ref{ex:subgaussian}) and III. Regular models (Ex. \ref{sec:regulairtyconditions}).
\end{itemize}

The rest of the paper is organized as follows. In Section \ref{sec:definiciones}, we define the general setting and recall some basics on generalization bounds along with some relevant results from the bibliography. In Section \ref{sec:pacmanintro} we introduce our analysis and we explain how this approach can be compared with the typical PAC point-wise framework. Several PAC-style bounds based on information metrics are presented in Section \ref{sec:mainresults}. Finally, the summary and conclusions are presented in Section \ref{sec:conclution}. Some supplementary material and discussions are relegated to the appendix. 

\subsection{Notation and conventions}

Probability measures and the expectation operator are denoted with the symbols $\mathbb{P}$ and $\mathbb{E}$ respectively. Boldsymbols denote vectors $\mathbf{v}$, where $(\mathbf{v})_i$ refers to the i-th element of vector $\mathbf{v}$. Information-theoretic metrics such as Kullback Leibler divergence, R\'enyi $\alpha$-divergence, square-chi divergence, Hellinger distance, total variation, mutual information and Sibson $\alpha$-mutual information and conditional entropy are defined as\footnote{In this work, we will typically assume the existence of probability density functions (pdf) with respect to some unspecified dominating measure. The pdf of random variable $X$ which takes values on space $\mathcal{X}$ will be denoted as $p(x)$, $x\in\mathcal{X}$.  This assumption, which do not severely restrict the results, is done in order to simplify the presentation of the results.}:
\begin{align}
\text{KL}(p(x)\|q(x))&=\mathop{\mathbb{E}}_{x\sim p(x)}\left[\log\frac{p(x)}{q(x)}\right],\\
D_\alpha(p(x)\|q(x))&=\left\{\begin{array}{cc}\displaystyle\frac{1}{\alpha-1}\log\mathop{\mathbb{E}}_{x\sim q(x)}\left[\left(\frac{p(x)}{q(x)}\right)^{\alpha}\right]&\quad\alpha\in(0,1)\cup(1,+\infty)\\\text{KL}(p(x)\|q(x))&\quad\alpha=1\end{array}\right.\\
\chi^2(p(x)\|q(x))&=\mathop{\mathbb{E}}_{x\sim p(x)}\left[\frac{p(x)-q(x)}{q(x)}\right],\\
\text{Hel}^2(p(x),q(x))&=2\left(1-\mathop{\mathbb{E}}_{x\sim p(x)}\left[\sqrt{\frac{q(x)}{p(x)}}\right]\right),\\
V(p(x),q(x))&=\mathop{\mathbb{E}}_{x\sim p(x)}\left[\frac{|p(x)-q(x)|}{p(x)}\right],\\
I(x;y)&=\mathop{\mathbb{E}}_{x\sim p(x)}\left[\textup{KL}(p(y|x)\|p(y))\right],\\
I_\alpha(x;y)&=\inf_{\tilde{p}(y)}\mathop{\mathbb{E}}_{x\sim p(x)}\left[D_\alpha(p(y|x)\|\tilde{p}(y))\right],\\
H(y|x)&=\mathop{\mathbb{E}}_{(x,y)\sim p(x,y)}\left[-\log p(y|x)\right].
\end{align}
We use $\log$ as natural logarithm and calligraphic letters to refer to spaces (i.e. $\mathcal{X}$), and $|\mathcal{X}|$ for their cardinality.

\section{Setting and Basics}\label{sec:definiciones}


\subsection{General Framework}
Let us consider an arbitrary input space $\mathcal{X}$ and
a finite output space $\mathcal{Y}=\{1,\cdots,|\mathcal{Y}|\}$. The examples $(x,y)\in\mathcal{X}\times\mathcal{Y}$ are input-output pairs; $x$ is a description, and $y$ is a label. We study the inductive learning setting where each random example $(x,y)$ is drawn i.i.d. from an unknown probability density distribution $p(x,y)$. Given a training set $S=\{(x_i,y_i)\}_{i=1}^n\sim p(S)\triangleq\prod_{(x,y)\in S}p(x,y)$, a learning algorithm builds a classifier that is later used to classify new examples drawn from $p(S)$. We consider a hypothesis set $\mathcal{H}$ of functions $h:\mathcal{X}\rightarrow\mathcal{Y}$. The learner aims to find $h\in\mathcal{H}$ that assign a label $y$ to a new input $x$ as correctly as possible. 

\subsection{Usual approach for the study of generalization}
Given an example $(x,y)\in\mathcal{X}\times\mathcal{Y}$ and a hypothesis $h\in\mathcal{H}$, we assess the quality of the
prediction of $h$ with a loss function $\ell:\mathcal{H}\times\mathcal{X}\times\mathcal{Y}\rightarrow[0,+\infty)$ that
evaluates to which extent the prediction is accurate. From
the loss $\ell$, we define the true risk $R_p^\ast(h)$ of $h\in\mathcal{H}$, which corresponds to the expected loss over the data distribution $p(S)$:
\begin{equation}
    R_p^\ast(h)\triangleq\mathop{\mathbb{E}}_{(x,y)\sim p(x,y)}\left[\ell(h,x,y)\right].
\end{equation}
The ultimate goal to find the hypothesis from $\mathcal{H}$ that minimizes the true risk which depends distribution $p(x,y)$. In order to cope with the lack of knowledge of $p(x,y)$, the Empirical Risk Minimization (ERM) principle is typically used. The empirical risk,
\begin{equation}
 R_S^\ast(h)=\frac1n\sum_{(x,y)\in S}\ell(h,x,y),
\end{equation}
is minimized with respect to $h\in\mathcal{H}$. That is, the same loss function is used in the definition of the true risk and the empirical risk. While this could be the case in some problems, it is not the case generally, specially for classification problems. In those problems  the ultimate true risk is the classification error probability, which is constructed with $\ell$ equal to the $0$-$1$ loss function. Unfortunately, that loss function is not typically used for the empirical risk because leads 
to a very complex and hard optimization problem. In any case, even if that optimization problem can be solved, we have no guarantee that such $h$ will generalize well on new unseen data. We thus need to bring theoretical guarantees to quantify how close the empirical risk to the true risk.  There are several different approaches that can be used for this problem \cite{DBLP:journals/corr/abs-2102-08649}.

\subsubsection{Uniform Convergence Bounds}
Uniform Convergence bounds consider a complexity measure of the hypothesis set $\mathcal{H}$ and stand for all the hypotheses of $\mathcal{H}$. Among the most renowned complexity measure, we can cite the VC-dimension or the Rademacher Complexity \cite{Shalev-Shwartz_Ben-David_2014}. This type of bound usually takes the form
\begin{equation}
    \mathop{\mathbb{P}}_{S\sim p(S)}\left(\sup_{h\in\mathcal{H}}|R_p^\ast(h)-R_S^\ast(h)|\leq \epsilon(\delta,n,\mathcal{H})\right)\geq 1-\delta.
\end{equation}
Put into words, with high probability $(1-\delta)$ on the random choice of $S$, we obtain good generalization properties when the worst-case of the deviation between the true risk $R_p^\ast(h)$ and its empirical estimate $R_S^\ast(h)$ ($\epsilon(\delta,n,\mathcal{H})$) is low. 

\subsubsection{PAC Bayesian Bounds}
A PAC-Bayesian generalization bound provides a bound in expectation/average over the hypothesis set $\mathcal{H}$ for any posterior measure $p(h|S)$ on $\mathcal{H}$ learned after the observation of $S$. The general form of such bounds is
\begin{equation}
    \mathop{\mathbb{P}}_{S\sim p(S)}\left(\forall p(h|S):\; \mathop{\mathbb{E}}_{h\sim p(h|S)}\left[|R_p^\ast(h)-R_S^\ast(h)|\right]\leq \epsilon(\delta,n,p(h|S))\right)\geq 1-\delta.
\end{equation}
Such PAC-Bayesian bounds are known to be tight, but they
stand for a randomized classifier by nature (due to the expectation over $\mathcal{H}$). A key issue for usual machine learning tasks is then the derandomization of the bounds to obtain a guarantee for a specific classifier \cite{Langford_Shawe_Taylor_2003}, \cite{Nagarajan_Kolter_2019}.

\subsubsection{Point-wise Bounds}

Given a practical algorithm, the minimization procedure defines a posterior probability of the hypothesis given the training set\footnote{In some sense we can also think of $p(h|S)$ as the representative of the randomness of the minimizing algorithm when setting some its hyper-parameters (e.g. initialization, noise injection techniques, etc.) } $p(h|S)$. 
The final model $h\sim p(h|S)$ and the generalization guarantees not longer consider the average with respect to $p(h|S)$ of the true and the empirical risk:
\begin{align}
    \mathop{\mathbb{P}}_{\substack{S\sim p(S)\\h\sim p(h|S)}}&\Big(|R_p^\ast(h)-R_S^\ast(h)|\leq \epsilon(\delta,n,p(h|S),p(S))\Big)\geq 1-\delta.\label{eq:genconmodulo}
\end{align}
It is pretty clear that (\ref{eq:genconmodulo}) implies: 
\begin{align}
    \mathop{\mathbb{P}}_{\substack{S\sim p(S)\\h\sim p(h|S)}}&\Big(R_p^\ast(h)\leq R_S^\ast(h)+ \epsilon(\delta,n,p(h|S),p(S))\Big)\geq 1-\delta.\label{eq:gensinmodulo}
\end{align}

Clearly equation (\ref{eq:gensinmodulo}) has a practical and operative connection with the concept of regularization. In order to avoid the overfitting generated by the ERM of $R_S^\ast(h)$, regularization techniques complement the mentioned minimization. As $R_p^\ast(h)=R_S^\ast(h)+(R_p^\ast(h)-R_S^\ast(h))$, it is clear that minimizing the expected risk $R_p^\ast(h)$ is equivalent to simultaneously minimizing the empirical risk $R_S^\ast(h)$ and the difference $R_p^\ast(h)-R_S^\ast(h)$. In conclusion, obtaining tight bounds of the form given by \eqref{eq:gensinmodulo} can have more importance from the practical point of view because $\epsilon(\delta,n,p(h|S),p(S))$ represents a potential regularization-term. In the following some known results in this line of \eqref{eq:genconmodulo} and \eqref{eq:gensinmodulo} will be presented.


\subsection{Some Known Point-wise bounds using information measures}\label{sec:resultadosbibliograficos}

Although there are several point-wise results that are worth mentioning we will only the following three.

\begin{theorem}\cite[Theo. 8]{pmlr-v83-bassily18a}\label{thm:bassily}
Let $\ell(h,z)=\mathds{1}{\{h(x)=y\}}$, for all $\delta\in(0,1]$:
\begin{equation}
    \mathop{\mathbb{P}}_{\substack{S\sim p(S)\\h\sim p(h|S)}}\left(|R_p^\ast(h)-R_S^\ast(h)|\leq \sqrt{\frac{{I}(S;h)+1+\delta}{2n\delta}}\right)\geq 1-\delta,    
\end{equation}
where ${I}(S;h)$ is the mutual information.
\end{theorem}
Theo. \ref{thm:bassily} shows that the amount of information that an algorithm uses is a natural and important quantity to study. A central idea that this paper revolves around is that a learning algorithm that only uses a small amount of information from its input sample will generalize well. The bound proposes a mutual information as a complexity measure with a $n^{-1/2}$ scaling.

\begin{theorem}\cite[Cor. 4]{Esposito2020RobustGV}\label{thm:esposito}
Let $\ell(h,z)=\mathds{1}{\{h(x)=y\}}$ and $\alpha>1$, for all $\delta\in(0,1]$:
\begin{equation}
    \mathop{\mathbb{P}}_{\substack{S\sim p(S)\\h\sim p(h|S)}}\left(|R_p^\ast(h)-R_S^\ast(h)|\leq \sqrt{\frac{1}{2n}\left({I}_\alpha(S;h)+\log(2)-\frac{\alpha}{\alpha-1}\log\delta\right)}\right)\geq 1-\delta,    
\end{equation}
where ${I}_\alpha(S;h)$ is the Sibson $\alpha$-mutual information.
\end{theorem}
Similar to the previous theorem, Theo. \ref{thm:esposito} proposes a Sibson $\alpha$-mutual information as a complexity measure with a $n^{-1/2}$ scaling. {As mentioned in \cite{Esposito2020RobustGV}, one important characteristic of these results is that they involve information-measures satisfying the data-processing
inequality \cite{6832827}. This means that all these results about generalization are robust to post-processing, i.e., if the outcome of any learning algorithm with bounded mutual information (classical or Sibson) is processed further, the value of the information measure cannot increase.}

\begin{theorem}\cite[Theo. 2]{DBLP:journals/corr/abs-2102-08649}\label{thm:viallard}
Let $\phi:\mathcal{H}\times(\mathcal{X}\times\mathcal{Y})^n\rightarrow(0,+\infty)$ and $\alpha>1$, for all $\delta\in(0,1]$ and any prior $\tilde{p}(h)$:
\begin{equation}
    \mathop{\mathbb{P}}_{\substack{S\sim P(S)\\h\sim p(h|S)}}\left(\log\phi(h,S)\leq{D}_\alpha(p(h|S)\|\tilde{p}(h))+\log\left[\mathop{\mathbb{E}}_{\substack{S^\prime\sim P(S)\\h^\prime\sim \tilde{p}(h)}}\left[\phi(h^\prime,S^\prime)\right]\left(\frac{2}{\delta}\right)^{1+\frac{\alpha}{\alpha-1}}\right]\right)\geq 1-\delta,    
\end{equation}
where ${D}_\alpha(\cdot\|\cdot)$ is the R\'enyi $\alpha$-divergence.
\end{theorem}

Finally, Theo. \ref{thm:viallard} shows a general inequality for a general functional $\phi(h,S)$. While standard result is to choose $\phi(h,S)=e^{\frac{2n\alpha}{\alpha-1}(R_p^\ast(h)-R_S^\ast(h))^2}$, this theorem is more flexible, because among other things it is not only valid for the $0-1$ loss function as the previous ones. Similarly as the other results, it shows that information-theoretic quantities (Rényi $\alpha$-divergence in this case) are important.

\section{PACMAN Analysis}\label{sec:pacmanintro}

PAC style bounds are usually defined for a specific loss function present both at training and testing phases. However in practice, specially for classification tasks, the loss used at training is not same that the one used at testing. It is typically the case that at the training phase the negative log-loss (i.e. cross entropy risk) is used in with the  ERM principle. This is due to the fact that this loss, jointly with a sufficient smooth parametrization of the hypothesis space leads to a risk function that can be efficiently solved using well-known procedures as backpropagation. However, the ultimate goal is the minimization of the missclassification probability, and it is typically this metric that is usually reported as performance metric in most of works. Missclassification probability (or accuracy) is linked with the $0$-$1$ loss, which is well-known to be difficult to be efficiently optimized during training using the ERM principle. In this way, in practice, there is a mismatch between the loss functions used at training and testing, that it is usually ignored in most works studying generalization for learning problems. In the following, we will study the generalization capabilities of learning algorithms using point-wise PAC bounds and considering explicitly the above mentioned mismatch. 


\subsection{Definitions}\label{sec:pacmandef}
Classification algorithms usually learn a vectorial function $\mathbf{h}:\mathcal{X}\rightarrow(0,1)^{|\mathcal{Y}|}$ before of final hypothesis $h(\cdot)$, where the ultimate decision is based on $h(x)=\mathop{\arg\max}_{y\in\mathcal{Y}}(\mathbf{h}(x))_y$. When the algorithm end with a sigmoid layer (binary classification) or soft-max layer (multi-label classification) this function meet $\sum_{y\in\mathcal{Y}}(\mathbf{h}(x))_y=1$ and its values can be interpreted as probabilities associated with the corresponding class labels\footnote{In fact, sigmoid layers return a single value, precisely because the remainder can be calculated by the complement.}. This interpretation arises that the classifier is chosen during the training stage minimizing the empirical cross-entropy on a finite training $n-$length set $S$:
\begin{equation}\label{eq:empce}
\text{CE}_{S}(\mathbf{h})=\frac{1}{n}\sum_{(x,y)\in S}-\log(\mathbf{h}(x))_{y},
\end{equation}
where the minimization is over an specific hypothesis space $\mathbf{h}\in\mathcal{H}$. Cross-entropy \eqref{eq:empce} is the empirical approximation of the true cross-entropy defined as:
\begin{equation}
\text{CE}_p(\mathbf{h})=\mathop{\mathbb{E}}_{(x,y)\sim p(x,y)}\left[-\log({\mathbf{h}}(x))_{y}\right].
\end{equation}
It is easy to show that:
\begin{equation}
    \text{CE}_p(\mathbf{h})=H(y|x)+\mathop{\mathbb{E}}_{x\sim p(x)}\left[\text{KL}\left(p(y|x)|({\mathbf{h}}(x))_{y}\right)\right],
\end{equation}
where $H(y|x)$ is the conditional entropy associated to  $p(y|x)$. This shows that $\text{CE}_p(\mathbf{h})$ is minimized when ${\mathbf{h}}(x)$ is the conditional class label probability given $x$, and the global infimum of $\text{CE}_p(\mathbf{h})$ is the conditional entropy. It is obvious that  when the hypothesis set $\mathcal{H}$ contains this conditional probability of the labels given the inputs this global optimum can be achievable.  

It is clear that ${\mathbf{h}}(x)$ can be interpreted as a stochastic or soft classifier, where $\hat{y}$ is defined as the random variable induced by the probability measure $\hat{y}\sim \mathbf{h}(x)$ with $(x,y)\sim p(x,y)$. It is also clear that, despite the above discussion about $\text{CE}_p(\mathbf{h})$ the most important risk to evaluate for a classifier (at least during the testing phase) is the probability of the error event $\{y\neq\hat{y}\}$, i.e. the probability that $\mathbf{h}$ misclassifies an example, which in this case can be written as:
\begin{align}
R_p(\mathbf{h})& = \mathop{\mathbb{P}}_{\substack{(x,y)\sim p(x,y)\\\hat{y}\sim\mathbf{h}(x)}}\Big(\hat{y}\neq y\Big)\\
&=1-\mathop{\mathbb{E}}_{\substack{(x,y)\sim p(x,y)\\\hat{y}\sim\mathbf{h}(x)}}\left[\mathds{1}\left\{y=\hat{y}\right\}\right]\\
&= 1-\mathop{\mathbb{E}}_{(x,y)\sim p(x,y)}\left[({\mathbf{h}}(x))_y\right].
\end{align}

This risk is globally minimized with the Bayes classifier \cite[Ex. 2.11]{Duda}, that is $({\mathbf{h}}(x))_y=$\\ $\mathds{1}\left\{y=\arg\max_{y^\prime\in\mathcal{Y}}p(y^\prime|x)\right\}$ which is clearly a deterministic (or hard classifier) and do not coincide, in general, with the optimal of $\text{CE}_p(\mathbf{h})$ as explained above. In a more practical setting, in which the hypothesis class $\mathcal{H}$ do not contain this optimal classifier, it is customary to use the $\mathbf{h}(x)$ obtained by the application of the ERM principle (using (\ref{eq:empce})), to construct a hard classification rule $h(x)$ by:
\begin{equation}
 h(x)=\mathop{\arg\max}_{y\in\mathcal{Y}}(\mathbf{h}(x))_y,
\end{equation} 
which will lead to the misclassification probability given by $R_p^\ast(h)\equiv\mathop{\mathbb{P}}_{(x,y)\sim p(x,y)}\big(h(x)\neq y\big)$. The mismatch in the loss function between the training and testing phase is clear: while at training we use $-\log(\mathbf{h}(x))_{y}$, at testing the loss function is $\mathds{1}\left\{h(x)\neq y\right\}$.

In this work we relax the problem by studying $R_p(\mathbf{h})$ instead of $R_p^\ast(h)$ as loss function at testing. This is because, the analysis of $R_p^\ast(h)$ could be harder. Later, we will see that between $\text{CE}_{S}(\mathbf{h})$ and $R_p(\mathbf{h})$ there is a nice relation that can be exploited, which is not the case when we consider $R_p^\ast(h)$.  On the other hand, from a practical point of view, we think that this is not a critical issue.  The empirical minimization of the cross-entropy where the last layer is soft-max tends to concentrate the probabilities \cite[Sec. 6.2.2]{Goodfellow-et-al-2016}, reducing the gap between both risks. This effect is also related to the fact that the algorithm learns through samples and empirical distribution is concentrated. On the other hand, this concentration effect is exacerbated by large neural networks and current regularization methods, such as batch normalization or weight decay, makes a calibration stage recommended if you want the soft-max output to be representative of the conditional probabilities \cite{guo17}. From this perspective, the approximation $R_p^\ast(h)\approx R_p(\mathbf{h})$ is relatively justified for modern neural architectures whose training is done using the negative log-loss.
In any case, in this work, if we would like to extend the results obtained to the case in which $R_p^\ast(h)$ is of interested at testing time, we just need to assume $R_p^\ast(h)\lesssim R_p(\mathbf{h})$ since we will look for upper bounds on $R_p(\mathbf{h})$ depending on $\text{CE}_{S}(\mathbf{h})$. {In Appendix \ref{sec:hardsoft} we study theoretical conditions to ensure this.}

As explained above, while the goal is to find guarantees over $R_p(\mathbf{h})$, the magnitude which is minimized during training is $\text{CE}_{S}(\mathbf{h})$. But the distinct nature of these magnitudes makes their comparison not immediate. We proposes to study a quantity that is a bijection of $R_p(\mathbf{h})$, defined as:
\begin{equation}
\widetilde{\text{CE}}_p(\mathbf{h})=-\log\left[1-R_p(\mathbf{h})\right].
\end{equation}
We will refer to this quantity as the cross-entropy of the expected risk.
In this way, we focus in style bounds using the cross-entropy of the expected risk, i.e. $\widetilde{\text{CE}}_p(\mathbf{h})\leq \text{CE}_{S}(\mathbf{h})+\text{BOUND}$ which are equivalent to $R_p(\mathbf{h})\leq 1-e^{-\text{CE}_{S}(\mathbf{h})-\text{BOUND}}\leq\text{CE}_{S}(\mathbf{h})+\text{BOUND}$ . We call this approach PACMAN (PAC-style bounds over the Mismatch of Accuracy and Negative log-loss) analysis. This analysis study the gap $\widetilde{\text{CE}}_p(\mathbf{h})-\text{CE}_{S}(\mathbf{h})$ as a generalization metric; in Appendix \ref{sec:signo} we present an auxiliary discussion about its behaviour.
As we will consider point-wise PAC bounds we will introduce a posterior distribution $p(\mathbf{h}|S)$.
Using the i.i.d. distribution of the training set $p(S)$ an a priori distribution $p(\mathbf{h})$ and a likelihood $p(S|\mathbf{h})$ can be obtained through Bayes rule.  
\begin{remark}
The main difference between the prior $p(\mathbf{h})$ and the classical prior $\tilde{p}(\mathbf{h})$  of a PAC bayesian bound is that $p(\mathbf{h})$ is defined with the data distribution $p(S)$ and the algorithm mapping $p(\mathbf{h}|S)$, while $\tilde{p}(\mathbf{h})$ is an auxiliary distribution which can be arbitrarly chosen. 
\end{remark}
In the following we introduce some ideas to compare PACMAN bounds with traditional bounds over the error probability gap.

\subsection{Relation between PACMAN and standard PAC point-wise bounds}\label{sec:comparison}

In the following sections we will obtain PAC point-wise bounds with the following form:
\begin{equation}\label{eq:aux_pacman}
\mathop{\mathbb{P}}_{\substack{S\sim p(S)\\\mathbf{h}\sim{p(\mathbf{h}|S)}}}\left(\widetilde{\textup{CE}}_p(\mathbf{h})\leq\textup{CE}_{S}(\mathbf{h})+\epsilon(\delta,n,p(\mathbf{h}|S),p(S))\right)\geq1-\delta.
\end{equation}
The values $\epsilon(\delta,n,p(\mathbf{h}|S),p(S))$ in those bounds will be called PACMAN bounds. The following lemma allow us to see how we can use them to obtain bounds on the risk $R_p(\mathbf{h})$:

\begin{lemma}\label{lem:paracomparar}
Let $\epsilon(\delta,n,p(\mathbf{h}|S),p(S))$ a PACMAN bound, i.e.
\begin{equation}
\mathop{\mathbb{P}}_{\substack{S\sim p(S)\\\mathbf{h}\sim{p(\mathbf{h}|S)}}}\left(\widetilde{\textup{CE}}_p(\mathbf{h})\leq\textup{CE}_{S}(\mathbf{h})+\epsilon(\delta,n,p(\mathbf{h}|S),p(S))\right)\geq1-\delta.
\end{equation}
Then
\begin{equation}\label{eq:comparar}
\mathop{\mathbb{P}}_{\substack{S\sim p(S)\\\mathbf{h}\sim{p(\mathbf{h}|S)}}}\left(R_p(\mathbf{h})\leq\left(1-e^{-\textup{CE}_{S}(\mathbf{h})}\right) + e^{-\textup{CE}_{S}(\mathbf{h})}\epsilon(\delta,n,p(\mathbf{h}|S),p(S))\right)\geq1-\delta.
\end{equation}
\end{lemma}

\begin{figure}[t]
    \centering\includegraphics[scale=0.6]{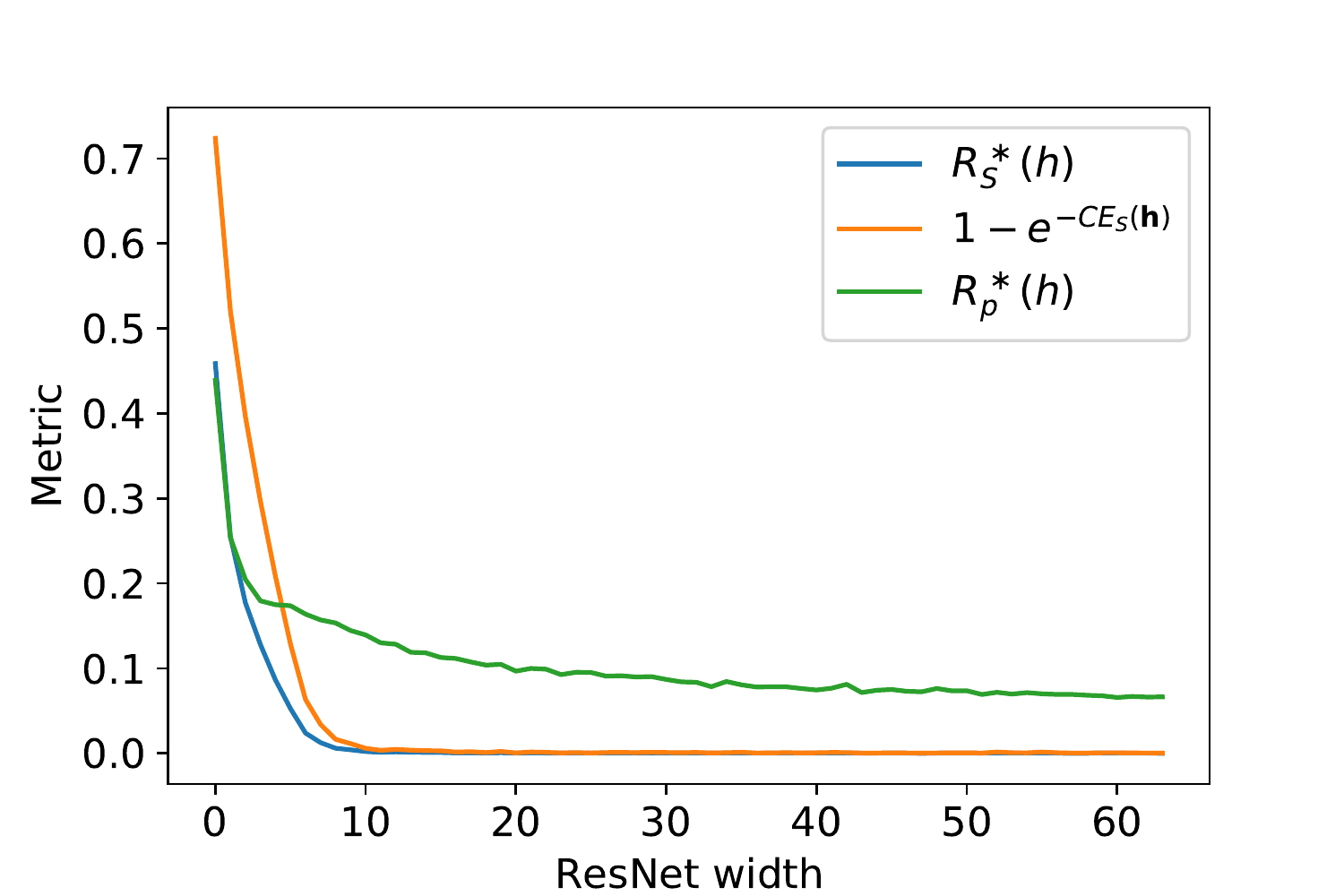}
    \caption{Comparison between cross-entropy bias $1-e^{-{\textup{CE}}_S(\mathbf{h})}$ (computed with the training set), empirical error probability $R_S^\ast(h)$ (computed with the training set) and expected error probability $R_p^\ast(h)$ (computed with the testing set) in a ResNet18 algorithm training with CIFAR-100.}
    \label{fig:bias}
\end{figure}
The proof is relegated to Appendix \ref{app:paracomparar}. It is important to note that in the above result we have: 1) a bias term smaller than the empirical cross entropy $1-e^{-{\textup{CE}}_S(\mathbf{h})}\leq \textup{CE}_S(\mathbf{h})$, 2) a term $e^{-{\textup{CE}}_S(\mathbf{h})}\leq1$ which weights the influence of $\epsilon(\delta,n,p(\mathbf{h}|S),p(S))$. In practice, for large and typically overparametrized networks the bias value will be very small.  Fig. \ref{fig:bias} shows the different metrics for a ResNet18 \cite{resnet} algorithm over CIFAR-100 dataset for different ResNet widths (layer  widths  for  the  4  blocks  are $[k,2k,4k,8k]$ for varying $k\in\mathbb
N$). Simulation uses the trained model presented in \cite{Nakkiran2020Deep}\footnote{Original codes of \cite{Nakkiran2020Deep} are available at \url{https://gitlab.com/harvard-machine-learning/double-descent/tree/master}.} (more details can be found in the original paper). For small networks there is a difference between the bias terms (with respect to $R_p^\ast(h)$), but for a reasonable widths the difference is not appreciable (standard ResNet18 corresponds to 64 convolutional channels in the first layer). Under these considerations, it is reasonable that PACMAN bounds of $\widetilde{\textup{CE}}_p(\mathbf{h})-\textup{CE}_{S}(\mathbf{h})$, i.e. $\epsilon(\delta,n,p(\mathbf{h}|S),p(S))$ are also bounds for $R_p(\mathbf{h})$, that is, with probability at least $1-\delta$:
\begin{equation}
  R_p^\ast(h)\lesssim R_p(\mathbf{h})\leq\textup{CE}_S(\mathbf{h})+
  \epsilon(\delta,n,p(\mathbf{h}|S),p(S))\approx\epsilon(\delta,n,p(\mathbf{h}|S),p(S)).
  \label{eq:generalization_rel}
\end{equation}

\section{PACMAN Bounds: Main Results}\label{sec:mainresults}

In this section we introduce different PACMAN bounds. The working methodology is similar to one using in other PAC point-wise bounds papers, but based on the analysis of equation \eqref{eq:aux_pacman}. We begin by considering $q(x,y|\mathbf{h})$, the conditional distribution of $(x,y)$ given the event $\{y=\hat{y}\}$ for a specific hypothesis $\mathbf{h}\in\mathcal{H}$. It is easy to show using Bayes rule: 
\begin{align}
q(x,y|\mathbf{h})&=\frac{\displaystyle\mathop{\mathbb{P}}_{\hat{y}\sim\mathbf{h}(x)}\left(\hat{y}=y\right)}{\displaystyle\mathop{\mathbb{P}}_{\substack{(x^\prime,y^\prime)\sim p(x,y)\\\hat{y}\sim\mathbf{h}(x^\prime)}}\Big(\hat{y}=y^\prime\Big)}p(x,y)=\frac{({\mathbf{h}}(x))_{y}}{1-R_p(\mathbf{h})}{p}(x,y)\label{eq:qmed}.
\end{align} 
 It is pretty clear that $q(x,y|\mathbf{h})$ is an absolutely continuous with respect to $p(x,y)$. It is important to note than $q(x,y|\mathbf{h})=p(x,y)$ if and only if $(\mathbf{h}(x))_y=\frac{1}{|\mathcal{Y}|}$ is constant for all $(x,y)$ which is a classifier that takes uniform and randomized decisions (also independent from the description $x$) about the label.  Consider  $q(S|\mathbf{h})=\prod_{(x,y)\in S} q(x,y|\mathbf{h})$ for $h\in\mathcal{H}$.  The difference between $\widetilde{\text{CE}}_p(\mathbf{h})$ and the empirical cross-entropy can be written as a likelihood ratio
\begin{equation}\label{eq:ratio}
\widetilde{\text{CE}}_p(\mathbf{h})-\text{CE}_{S}(\mathbf{h})=\frac{1}{n}\sum_{(x,y)\in S}\log\frac{({\mathbf{h}}(x))_{y}}{1-R_p(\mathbf{h})}=\frac{1}{n}\log\frac{q(S|\mathbf{h})}{p(S)}.
\end{equation}


Equation (\ref{eq:ratio}) is conceptually interesting. If $\frac{1}{n}\log\frac{q(S|\mathbf{h})}{p(S)}$ is large, that would be indicative a poor generalization capabilities (see Lemma \ref{lem:paracomparar}). This would be the case if the distribution $q(S|\mathbf{h})$ is ``more likely'' that $p(S)$ when observing the training set $S$ and with $\mathbf{h}$ chosen using $S$ (as a learning algorithm would do). A more likely $q(S|\mathbf{h})$, which means that the chosen $\mathbf{h}$ is correctly classifying the training data $S$, could be linked to the idea that overfitting is present. In section \ref{sec:LR} we will use \eqref{eq:ratio} to make some connections with a Neyman-Pearson decision problem \cite[Theorem 11.7.1]{cover}.

In the following we will work with (\ref{eq:ratio}) in order to obtain PACMAN bounds. We will see how different information-theoretic quantities play a significant role in those bounds. We will also compare our results with some results from the literature. 



\subsection{A family of bounds using R\'enyi divergences}\label{sec:renyi}

Equation \eqref{eq:ratio} can be written as:
\begin{align}
\widetilde{\text{CE}}_p(\mathbf{h})-\text{CE}_{S}(\mathbf{h})&=\frac{1}{n}\log\frac{q(S|\mathbf{h})}{p(S|\mathbf{h})}+\frac{1}{n}\log\frac{p(S|\mathbf{h})}{p(S)}\label{eq:2ratio-bis}\\
&=\frac{1}{n}\log\frac{q(S|\mathbf{h})}{p(S|\mathbf{h})}+\frac{1}{n}\log\frac{p(\mathbf{h}|S)}{p(\mathbf{h})}.\label{eq:2ratio}
\end{align}
The first term of \eqref{eq:2ratio} consider, conditional to a classifier $\mathbf{h}\in\mathcal{H}$, the likelihood of the distribution of every sample with respect to the distribution of the well-classified samples. The second term the mutual information density rate \cite{Han_2002} of $(\mathbf{h},S)$, which, in some sense measure the common information content between $\mathbf{h}$ and $S$. We have the following theorem:


\begin{theorem}\label{thm:bayes}
Let $\alpha\in(0,1)\cup(1,+\infty)$ and $\beta\in(1,+\infty)$, for every $\delta\in(0,1]$ we have,
\begin{align}
    \mathop{\mathbb{P}}_{\substack{S\sim p(S)\\\mathbf{h}\sim p(\mathbf{h}|S)}}&\left(\widetilde{\textup{CE}}_p(\mathbf{h})\leq\textup{CE}_{S}(\mathbf{h})+\frac{\alpha-1}{\alpha n}D_\alpha(q(S|\mathbf{h})\|p(S|\mathbf{h}))+\frac{D_\beta(p(\mathbf{h}|S)\|p(\mathbf{h}))}{n}+\right.\nonumber\\
&\qquad\left.+\left(\frac{1}{\alpha n}+\frac{1}{(\beta-1)n}\right)\log\frac{2}{\delta}\right)\geq 1-\delta.\label{eq:bayes}    
\end{align}

\end{theorem}

The proof is relegated to the Appendix \ref{app:bayes}. As (\ref{eq:2ratio}) suggests, we see that the PACMAN bound is separated on two factors: the first, $D_\alpha(q(S|\mathbf{h})\|p(S|\mathbf{h}))$, it is related with the ability of the selected classifier $\mathbf{h}\in\mathcal{H}$ of classifying correctly the samples in $S$ (which are also used for obtaining $\mathbf{h}$). As R\'enyi divergences are positive, an interesting case is when $\alpha<1$, where we see that this term is multiplied by a negative number and could have some beneficial effect. Using that \cite{erven14}:
\begin{equation}
\frac{\alpha}{2}V^2(q(S|\mathbf{h}),p(S|\mathbf{h}))\leq D_\alpha(q(S|\mathbf{h})\|p(S|\mathbf{h})),\quad \forall\; \mathbf{h}\in\mathcal{H},\ \alpha\in(0,1]
    \label{eq:Pinsker}
\end{equation}
where $V(\cdot,\cdot)$ is the total variation distance we see that if the distribution of samples $S$ conditional to the obtained classifier $\mathbf{h}$ is not similar to the distribution of the samples conditioned on the event that are correctly classified using the same $\mathbf{h}$. Qualitatively, when $q(S|\mathbf{h})$ and $p(S|\mathbf{h})$ are not close to each other, a good learning algorithm should not totally look for the correct classification of the training data set which can be see as an overfitting mitigation. The other term, is easier to interpret:  when $p(\mathbf{h}|S)$ is not very different to $p(\mathbf{h})$ the penalization that it introduce will be smaller. This is linked with the well known fact that good generalization capabilities are related with situations in which the output of the learning algorithm is not extremely sensitive to input $S$.
The case with $\alpha=\frac{1}{2}$ and $\beta=2$ is particularly interesting. We have the following corollary (whose proof can be found in Appendix \ref{app:bayes})

\begin{corollary}\label{cor:1}
When $\alpha=\frac{1}{2}$ and $\beta=2$, (\ref{eq:bayes}) can be written as:
	\begin{equation}\label{eq:hellchi}
	\mathop{\mathbb{P}}_{\substack{S\sim p(S)\\\mathbf{h}\sim{p(\mathbf{h}|S)}}}\left(\widetilde{\textup{CE}}_p(\mathbf{h})\leq\textup{CE}_{S}(\mathbf{h})+\frac{2}{n}\log\left(1-\frac{H(\mathbf{h})}{2}\right)+\frac{1}{n}\log\left(1+\textup{CS(S)}\right)+\frac{3}{n}\log\frac{2}{\delta}\right)\geq1-\delta,
	\end{equation}
	where 
	\begin{align}\label{eq:H-CSdefs}
	H(\mathbf{h})&=\textup{Hel}^2(q(S|\mathbf{h}),p(S|\mathbf{h})),\qquad \textup{CS}(S)=\chi^2(p(\mathbf{h}|S)\|p(\mathbf{h})).
	\end{align}
\end{corollary}
In this case, in line with the above explanation, $\textup{CS}(S)$ is an information-theoretic quantity, similar to others in the bibliography, that represents the overfitting effect generated when the choice of $\mathbf{h}$ is sensitive to the training set $S$. The term $H(\mathbf{h})$ penalizes classifiers which classifies correctly all the training set. It is clear that the worst-case for the $\frac{2}{n}\log\left(1-\frac{H(\mathbf{h})}{2}\right)$ term is when $H(\mathbf{h})=0$, which would implies overfitting with respect to $S$. However, if even in that case, the remaining term $\textup{CS}(S)$ is small, good generalization capabilities could be achieved. In fact, using Theo. \ref{thm:viallard} we can obtain the following corollary: 

\begin{corollary}\label{cor:viallard}
Let $\beta>1$, for all $\delta\in(0,1]$ and any prior $\tilde{p}(\mathbf{h})$:
\begin{equation}
    \mathop{\mathbb{P}}_{\substack{S\sim p(S)\\\mathbf{h}\sim p(\mathbf{h}|S)}}\left(\widetilde{\textup{CE}}_p(\mathbf{h})\leq\textup{CE}_{S}(\mathbf{h})+\frac1n\left[{D}_\beta(p(\mathbf{h}|S)\|\tilde{p}(\mathbf{h}))+\left(1+\frac{\beta}{\beta-1}\right)\log\frac2\delta\right]\right)\geq 1-\delta.    
\end{equation}
In particular for $\beta=2$ and $\tilde{p}(\mathbf{h})={p}(\mathbf{h})$
\begin{equation}\label{eq:chiviallard}
    \mathop{\mathbb{P}}_{\substack{S\sim p(S)\\\mathbf{h}\sim p(\mathbf{h}|S)}}\left(\widetilde{\textup{CE}}_p(\mathbf{h})\leq\textup{CE}_{S}(\mathbf{h})+\frac1n\log\left(1+\textup{CS}(S))\right)+\frac3n\log\frac2\delta\right)\geq 1-\delta.    
\end{equation}
\end{corollary}
\begin{proof}
Using $\phi(\mathbf{h},S)=e^{n\left(\widetilde{\textup{CE}}_p(\mathbf{h})-\textup{CE}_{S}(\mathbf{h})\right)}$, \eqref{eq:ratio} implies  $\displaystyle\mathop{\mathbb{E}}_{S\sim p(S)}\left[\phi(\mathbf{h},S)\right]=1$ for all $\mathbf{h}\in{\mathcal{H}}$ and the corollary is proved.
\end{proof}

We see that the PACMAN bound in Theo. \ref{thm:bayes}, with $\alpha<1$, is tighter than the one obtained in Corollary \ref{cor:viallard}, showing that our analysis is able to recover the negative term  $\frac{2}{n}\log\left(1-\frac{H(\mathbf{h})}{2}\right)$. 

Consider now Theo. \ref{thm:esposito}. Using the data-processing inequality $I_\alpha(S;\mathbf{h})\geq I_\alpha(S;h)$, which is valid because $h(x)=f(\mathbf{h}(x))$ \cite{6832827}, we can write:
\begin{align}
\mathop{\mathbb{P}}_{\substack{S\sim p(S)\\\mathbf{h}\sim p(\mathbf{h}|S)}}&\left(R_p^\ast(h)\leq R_S^\ast(h)+ \sqrt{\frac{1}{2n}\left({I}_\alpha(S;\mathbf{h})+\log(2)-\frac{\alpha}{\alpha-1}\log\delta\right)}\right)\geq 1-\delta.\label{eq:cota1comp}\\   
\end{align}
Choosing $\alpha=2$ it is easy to see that:
\begin{align}
   I_2(S;\mathbf{h})&=\min_{\tilde{p}(\mathbf{h})}\mathop{\mathbb{E}}_{S\sim p(S)}\left[D_2(p(\mathbf{h}|S)\|\tilde{p}(\mathbf{h}))\right]\\
   &=\min_{\tilde{p}(\mathbf{h})}\mathop{\mathbb{E}}_{S\sim p(S)}\left[\log(1+\chi^2\left(p(\mathbf{h}|S)\|\tilde{p}(\mathbf{h})\right))\right]\\
  &\leq \mathop{\mathbb{E}}_{S\sim p(S)}\left[\log(1+\textup{CS}(S))\right].
\end{align}

We see that in the setting of Theo. \ref{thm:esposito}, where the same loss function is used at the training and testing phases the term $\textup{CS}(S)$ have a similar influence. However, there are some differences. In \eqref{eq:hellchi} the actual value of $\textup{CS}(S)$, and not its expectation, is part of the bound\footnote{It is not difficult to see that with a simple modification of the arguments given in Appendix \ref{app:bayes} the term $\mathop{\mathbb{E}}_{S\sim p(S)}\left[\log(1+\textup{CS}(S))\right]$ can be included in the right-hand side of \eqref{eq:hellchi}.}.  Besides that, we see that that in \eqref{eq:hellchi} a better scaling with $n$ is obtained. Finally, in \eqref{eq:hellchi} we have the presence of the negative term $\frac{2}{n}\log\left(1-\frac{H(\mathbf{h})}{2}\right)$.


\subsection{Learners with Little Information}\label{sec:little}

Theo. \ref{thm:bassily} from \cite{pmlr-v83-bassily18a} shows that learners which share little information with the sample $S$ generalizes well when the $0-1$ loss is considered at training and testing. Let us consider the following lemma:
	\begin{lemma}\cite[Lemma 9]{pmlr-v83-bassily18a}\label{lem:littleinformation}
		Let us consider $p(x)$, $q(x)$ the induced distributions from the random variable $x$ and a measurable event $E$, then $\displaystyle \mathop{\mathbb{P}}_{x\sim p(x)}(x\in E) \leq \frac{\textup{KL}(p(x)\|q(x))+1}{-\log \mathop{\mathbb{P}}_{x\sim q(x)}(x\in E)}$.
	\end{lemma}	
Using this lemma we can obtain another PACMAN bound as follows:

\begin{theorem}\label{thm:little}
	For every $\delta\in(0,1]$ we have,
	\begin{equation}\label{eq:littleinf}
	\mathop{\mathbb{P}}_{\substack{S\sim p(S)\\\mathbf{h}\sim{p(\mathbf{h}|S)}}}\left(\widetilde{\textup{CE}}_p(\mathbf{h})\leq\textup{CE}_{S}(\mathbf{h})+\frac{I(S;\mathbf{h})+1}{n\delta}\right)\geq1-\delta.
	\end{equation}
\end{theorem}
\begin{proof}
We can use Lemma \ref{lem:littleinformation} to write:
	\begin{align}
	\mathop{\mathbb{P}}_{S\sim{p(S|\mathbf{h})}}\left(\widetilde{\textup{CE}}_p(\mathbf{h})-\textup{CE}_{S}(\mathbf{h})>\epsilon\right)\leq\frac{\textup{KL}(p(S|\mathbf{h})\|p(S))+1}{\displaystyle-\log \mathop{\mathbb{P}}_{S\sim{p(S)}}\left(\widetilde{\textup{CE}}_p(\mathbf{h})-\textup{CE}_{S}(\mathbf{h})>\epsilon\right)}.\label{eq:little}
	\end{align}
	Using Chernoff bound we can write:
	\begin{align*}
	    \mathop{\mathbb{P}}_{S\sim{p(S)}}\left(\widetilde{\textup{CE}}_p(\mathbf{h})-\textup{CE}_{S}(\mathbf{h})>\epsilon\right)&\leq\inf_{t>0} e^{-t\epsilon}\mathop{\mathbb{E}}_{S\sim{p(S)}}\left[\left(\frac{q(S|\mathbf{h})}{p(S)}\right)^{t/n}\right],
	\end{align*}
and re-bounding with $t=n$ we get:
	\begin{align}
\mathop{\mathbb{P}}_{S\sim{p(S)}}\left(\widetilde{\textup{CE}}_p(\mathbf{h})-\textup{CE}_{S}(\mathbf{h})>\epsilon\right)&\leq e^{-n\epsilon}\mathop{\mathbb{E}}_{S\sim{p(S)}}\left[\frac{q(S|\mathbf{h})}{p(S)}\right]=e^{-n\epsilon}.
\label{eq:bound_aux}
	\end{align}
Using (\ref{eq:little}) in (\ref{eq:bound_aux}) and taking expectation in \eqref{eq:little} over $\mathbf{h}\sim p(\mathbf{h})$:
	\begin{align}
	\mathop{\mathbb{P}}_{\substack{\mathbf{h}\sim p(\mathbf{h})\\S\sim{p(S|\mathbf{h})}}}\left(\widetilde{\textup{CE}}_p(\mathbf{h})-\textup{CE}_{S}(\mathbf{h})>\epsilon\right)&\leq\frac{I(S;\mathbf{h})+1}{n\epsilon}=\delta,
	\end{align}
	from which the result is straightforward.
\end{proof}

Theo. \ref{thm:bassily} can be relaxed using the data-processing inequality where $I(S;\mathbf{h})\geq I(S;h)$ which is valid because $h(x)=f(\mathbf{h}(x))$ \cite{6832827}:
\begin{align}
\mathop{\mathbb{P}}_{\substack{S\sim p(S)\\\mathbf{h}\sim p(\mathbf{h}|S)}}&\left(R_p^\ast(h)\leq R_S^\ast(h)+ \sqrt{\frac{{I}(S;\mathbf{h})+1+\delta}{2n\delta}}\right)\geq 1-\delta.\label{eq:cota2comp}
\end{align}

Despite the fact that (\ref{eq:cota2comp}) considers the true and the empirical risks of the $0-1$ loss and (\ref{eq:littleinf}) is PACMAN bound for risk $R_p(\mathbf{h})$ (and empirical negative log-loss at training), both show that mutual information  ${I}(S;\mathbf{h})$ has a important effect on generalization. The main difference is that (\ref{eq:littleinf}) presents a faster scaling $n^{-1}$ instead of $n^{-1/2}$.

The following lemma show mutual information ${I}(S;\mathbf{h})$ can be used to bound the expectation of $\widetilde{\textup{CE}}_p(\mathbf{h})-\textup{CE}_{S}(\mathbf{h})$:

\begin{lemma}\label{lem:esperanzaesinformacion}
\begin{align}\label{eq:mutualinformation}
\mathop{\mathbb{E}}_{\substack{S\sim p(S)\\\mathbf{h}\sim{p(\mathbf{h}|S)}}}&\left[\widetilde{\textup{CE}}_p(\mathbf{h})-\textup{CE}_{S}(\mathbf{h})\right]\leq\frac{I(S;\mathbf{h})}{n}.
\end{align}
\end{lemma}
\begin{proof}
Fixing $\mathbf{h}\in\mathcal{H}$, we can take expectation with respect to $p(S|\mathbf{h})$ in \eqref{eq:2ratio-bis} to obtain:
\begin{align}
    \mathop{\mathbb{E}}_{S\sim p(S|\mathbf{h})}&\left[\widetilde{\textup{CE}}_p(\mathbf{h})-\textup{CE}_{S}(\mathbf{h})\right]=\frac{\textup{KL}(p(S|\mathbf{h})\|p(S))-\textup{KL}(p(S|\mathbf{h})\|q(S|\mathbf{h}))}{n}\leq\frac{\textup{KL}(p(S|\mathbf{h})\|p(S))}{n}.
\end{align}
Finally taking expectation with respect to $p(\mathbf{h})$  the lemma is proved.
\end{proof}

In \cite[Theorem 1]{Xu_Raginsky_2017} it is reported that when the loss funciton (the same at training and testing) is $\sigma$-subgaussian for all $\mathbf{h}\in{\mathcal{H}}$, the expectation of the generalization error can be bounded $\sqrt{\frac{2\sigma^2}{n}I(S;\mathbf{h})}$. In our case, training and testing loss functions are not same. Besides that, we are not considering $|\widetilde{\textup{CE}}_p(\mathbf{h})-\textup{CE}_{S}(\mathbf{h})|$ but $\widetilde{\textup{CE}}_p(\mathbf{h})-\textup{CE}_{S}(\mathbf{h})$ instead. However, we obtain a faster scaling and our result does not depend on any assumption on the distribution of the loss functions we are considering both at training and testing.


\subsection{About the likelihood ratio $\frac{1}{n}\log\frac{q(S|\mathbf{h})}{p(S)}$}
\label{sec:LR}

In this section we will explore equation \eqref{eq:ratio} and show, among other thing how the problem of analyzing $\widetilde{\textup{CE}}_p(\mathbf{h})-\textup{CE}_{S}(\mathbf{h})$ is related with a Neyman-Pearson hypothesis testing problem \cite{cover}.
\begin{theorem}\label{thm:ratio}
	For every $\delta\in(0,1]$, we have	
	\begin{equation}\label{eq:teochernoff}
	\mathop{\mathbb{P}}_{\substack{S\sim p(S)\\\mathbf{h}\sim{p(\mathbf{h}|S)}}}\left(\widetilde{\textup{CE}}_p(\mathbf{h})\leq\textup{CE}_{S}(\mathbf{h})+\inf_{t>0}\frac1t\log\frac{C_n(t)}{\delta}\right)\geq1-\delta,
	\end{equation}
	where $\displaystyle C_n(t)=\mathop{\mathbb{E}}_{\substack{S\sim p(S)\\\mathbf{h}\sim p(\mathbf{h}|S)}}\left[\left(\frac{q(S|\mathbf{h})}{p(S)}\right)^{t/n}\right]$.
\end{theorem}
It is easy to see from \eqref{eq:ratio} that $C_n(t)$ is the moment-generating function (MGF) of the generalization metric $\widetilde{\textup{CE}}_p(\mathbf{h})-\textup{CE}_{S}(\mathbf{h})$. 

\begin{proof}
	Using Chernoff inequality, for every $t>0$:
	\begin{equation}\label{eq:chernoff}
	\mathop{\mathbb{P}}_{\substack{S\sim p(S)\\\mathbf{h}\sim{p(\mathbf{h}|S)}}}\left(\widetilde{\textup{CE}}_p(\mathbf{h})-\textup{CE}_{S}(\mathbf{h})>\epsilon\right)\leq e^{-t\epsilon}C_n(t).
	\end{equation}	
	The result is obtained calling $\delta=e^{-t\epsilon}C_n(t)$ and solving.
\end{proof}

\begin{remark} PACMAN generalization metric \eqref{eq:ratio} can be re-written using $C_n(n)$ as:
\begin{equation}\label{eq:ratioconcn}
\widetilde{\text{CE}}_p(\mathbf{h})-\text{CE}_{S}(\mathbf{h})=\frac{1}{n}\log\frac{q(S|\mathbf{h})}{p(S)}=\frac{1}{n}\log\frac{\tilde{q}(S,\mathbf{h})}{p(S)p(\mathbf{h}|S)}+\frac{1}{n}\log C_n(n),
\end{equation}
where $\tilde{q}(S,\mathbf{h})$ is a probability law defined as:
\begin{equation}
\tilde{q}(S,\mathbf{h})=\frac{q(S|\mathbf{h})p(\mathbf{h}|S)}{C_n(n)}.   
\end{equation}
The likelihood ratio introduced in \eqref{eq:ratioconcn} is related to a Neyman-Pearson hypothesis testing over the random variables $(S,\mathbf{h})$ \cite[Theorem 11.7.1]{cover}. Consider the following hypothesis test:
\begin{equation}
H_0:\;p(S)p(\mathbf{h}|S),\ \ \ \ \ H_1:\tilde{q}(S,\mathbf{h})
\label{eq:test1}
\end{equation}
It is well known that the uniformly most powerful test for the above testing problem is:
\begin{equation}
\widetilde{\text{CE}}_p(\mathbf{h})-\text{CE}_{S}(\mathbf{h})-\frac1n\log C_n(n)\underset{H_0}{ \overset{H_1}{\gtrless}}\epsilon_\alpha.
\label{eq:test2}    
\end{equation}
In this context the associated type I error $\alpha$ is:
\begin{equation}
\alpha=\mathop{\mathbb{P}}_{\substack{S\sim p(S)\\\mathbf{h}\sim{p(\mathbf{h}|S)}}}\left(\widetilde{\textup{CE}}_p(\mathbf{h})>\textup{CE}_{S}(\mathbf{h})+\frac1n\log{C_n(n)}+\epsilon_\alpha\right),
\end{equation}
which matches with Theo. \ref{thm:ratio} for $t=n$ and $\epsilon_\alpha=-\frac1n\log\delta$; and these choices guarantee $\alpha\leq\delta$. 
\end{remark}

In the following we present three special settings where the quantity $\displaystyle\left(\inf_{t>0}\frac1t\log\frac{C_n(t)}{\delta}\right)$ can be bounded from above and an explicit dependence of the bound with the size $n$ (typically $n^{-1}$) of the training sample can be computed.

\begin{example}[Finite hypothesis space]\label{cor:finito}
A simple case to analyze is when ${\mathcal{H}}$ is finite. In this case, using that  $p(\mathbf{h}|S)\leq1$, we obtain $C_n(n)\leq|{\mathcal{H}}|$. Then, choosing $t=n$ the following bound is obtained immediately:
	\begin{equation}
	\mathop{\mathbb{P}}_{\substack{S\sim p(S)\\\mathbf{h}\sim{p(\mathbf{h}|S)}}}\left(\widetilde{\textup{CE}}_p(\mathbf{h})\leq\textup{CE}_{S}(\mathbf{h})+\frac1n\log\frac{|{\mathcal{H}}|}{\delta}\right)\geq1-\delta.
	\end{equation}
This is result is in line with several well-known classic results \cite{Shalev-Shwartz_Ben-David_2014} where suggest that less complex models will be more robust to overfitting. 
\end{example}

\begin{example}[Subgaussian hypothesis on the generalization metric]\label{ex:subgaussian}
In Sections \ref{sec:renyi} and \ref{sec:little} we show that the scaling on the generalization gap $\widetilde{\textup{CE}}_p(\mathbf{h})-\textup{CE}_{S}(\mathbf{h})$ is $n^{-1}$ (with high probability). In this example we will further assume that $n\left(\widetilde{\textup{CE}}_p(\mathbf{h})-\textup{CE}_{S}(\mathbf{h})\right)$ is $\sigma$-subgaussian. That is,
\begin{equation}
\log\mathop{\mathbb{E}}_{\substack{S\sim p^n\\\mathbf{h}\sim p_{\mathbf{h}|S}}}\left[e^{\lambda n\left(\widetilde{\textup{CE}}_p(\mathbf{h})-\textup{CE}_{S}(\mathbf{h})\right)}\right]\leq\lambda n\mathop{\mathbb{E}}_{\substack{S\sim p^n\\\mathbf{h}\sim p_{\mathbf{h}|S}}}\left[\widetilde{\textup{CE}}_p(\mathbf{h})-\textup{CE}_{S}(\mathbf{h})\right]+\frac{\lambda^2\sigma^2}{2},\qquad\forall\;\lambda\in\mathbb{R}.
\end{equation}
For $\lambda>0$, the first expectation can be related with $C_n(t)$ using \ref{eq:ratio}:
\begin{equation}
C_n(\lambda n)=\mathop{\mathbb{E}}_{\substack{S\sim p^n\\\mathbf{h}\sim p_{\mathbf{h}|S}}}\left[e^{\lambda n\left(\widetilde{\textup{CE}}_p(\mathbf{h})-\textup{CE}_{S}(\mathbf{h})\right)}\right].
\end{equation}
Then we can write:
	\begin{align}
	\inf_{t>0}\frac1t\log\frac{C_n(t)}{\delta}&=	\inf_{\lambda>0}\frac{1}{\lambda n}\left[\log C_n(\lambda n)+\log(1/\delta)\right]\\
	&\leq\mathop{\mathbb{E}}_{\substack{S\sim p^n\\\mathbf{h}\sim p_{\mathbf{h}|S}}}\left[\widetilde{\textup{CE}}_p(\mathbf{h})-\textup{CE}_{S}(\mathbf{h})\right]+\inf_{\lambda>0}\frac{\lambda\sigma^2}{2n}+\frac{\log(1/\delta)}{\lambda n}\\
	&\leq\frac{I(S;\mathbf{h})}{n}+\frac{1}{n}\inf_{\lambda>0}\frac{\lambda\sigma^2}{2}+\frac{\log(1/\delta)}{\lambda},
	\end{align} 
where we have used Lemma \ref{lem:esperanzaesinformacion}. Using that $ax+\frac{b}{x}\geq2\sqrt{ab}$ for $a,b,x>0$ (with equality iff $x=\sqrt{b/a}$): 	
	\begin{equation}
	\mathop{\mathbb{P}}_{\substack{S\sim p(S)\\\mathbf{h}\sim{p(\mathbf{h}|S)}}}\left(\widetilde{\textup{CE}}_p(\mathbf{h})\leq\textup{CE}_{S}(\mathbf{h})+\frac{I(S;\mathbf{h})}{n}+\frac{1}{n}\sqrt{2\sigma^2\log(1/\delta)}\right)\geq1-\delta.
	\end{equation}
{This result, as Theo. \ref{thm:little} and Theo. \ref{thm:bassily} or \cite[Theorem 1]{Xu_Raginsky_2017} reflect the importance of $I(S;\mathbf{h})$ . However, in this example, although the scaling in the mutual information term is $n^{-1}$ as in Theo. \ref{thm:little}, the dependence with respect to $\delta$ is better behaved. This is a consequence of the $\sigma$-subgaussianity assumption. }
\end{example}

\begin{example}[Regularity conditions on the hypothesis class]\label{sec:regulairtyconditions} 
If we assume that there exists a $\mathbf{h}_0\in{\mathcal{H}}$ such that $({\mathbf{h}_0}(x))_{y}=\frac{1}{|\mathcal{Y}|}$ for all $(x,y)\in\mathcal{X}\times\mathcal{Y}$, the likelihood ratio in the definition of $C_n(t)$ can be bounded as:
\begin{equation}\label{eq:MLRdef}
\frac{q(S|\mathbf{h})}{p(S)}\leq\frac{\displaystyle\sup_{\mathbf{h}\in{\mathcal{H}}}q(S|\mathbf{h})}{\displaystyle q(S|\mathbf{h}_0)}=\frac{\displaystyle\sup_{\mathbf{h}\in{\mathcal{H}}}q(S|\mathbf{h})}{\displaystyle p(S)}.
\end{equation}
The numerator can be seen as the result of a likelihood maximization over the hypothesis space $\mathcal{H}$. Maximum likelihood estimators have a Gaussian behavior under some regularity conditions for the family of distributions $q(S|\mathbf{h}),\ \mathbf{h}\in\mathcal{H}$ \cite[Theo. 10.3.3]{CaseBerg:01}. Most of them apply to many common parametric families of distributions. However, one of the fundamental hypotheses is identifiability \cite{Vaart_2000} , which is not necessarily true in typically overparametrized deep networks. However, in many parametric settings this hypothesis is fullfilled. In this example we assume that the regularity conditions hold for $q(S|\mathbf{h})$ with $\mathbf{h}\in{\mathcal{H}}$.

Let the parametric space ${\mathcal{H}}=\left\{\mathbf{h}_\theta:\;\theta\in\Theta\right\}$ with $\Theta\subset\mathbb{R}^{\nu}$. Under appropriate regularity conditions, maximum likelihood ratio $X=2\log\frac{\displaystyle \sup_{\mathbf{h}\in{\mathcal{H}}}q(S|\mathbf{h})}{p(S)}$ has an asymptotically squared-chi distribution of $\nu$ degree freedom (the number of trainable parameters). Then $C_n(t)\leq\mathbb{E}_{X\sim \chi^2}\left[e^{\frac{Xt}{2n}}\right]$ is bounded by the squared-chi MGF evaluated at $\frac{t}{2n}$. For $0<t<n$, the squared-chi distribution has a MGF of $\left(1-\frac{t}{n}\right)^{-\frac{\nu}{2}}$. That is:
	\begin{equation}
\frac{C_n(t)}{\left(1-\frac{t}{n}\right)^{-\frac{\nu}{2}}}\leq\frac{\displaystyle\mathop{\mathbb{E}}_{S\sim p(S)}\left(\frac{\displaystyle \sup_{\mathbf{h}\in{\mathcal{H}}}q(S|\mathbf{h})}{p(S)}\right)^{t/n}}{\left(1-\frac{t}{n}\right)^{-\frac{\nu}{2}}}\overset{n\rightarrow\infty}{\longrightarrow} 1,
\end{equation}  
i.e. $C_n(t)\lesssim\left(1-\frac{t}{n}\right)^{-\frac{\nu}{2}}$, where $\lesssim$ refers to the asymptotic behaviour ($a_n\lesssim b_n$ if for all $\epsilon>0$ there exists $n_0$ such that for all $n\geq n_0$, $a_n\leq b_n+\epsilon$). We can write:
	\begin{align}
	\inf_{0<t<n}\frac1t\log\left[\frac{\left(1-\frac{t}{n}\right)^{-\frac{\nu}{2}}}{\delta}\right]&=\frac{1}{n}\inf_{0<\lambda<1}\frac1\lambda\log\left[\frac{\left(1-\lambda\right)^{-\frac{\nu}{2}}}{\delta}\right]\\
	&=\frac{g(\delta,\nu)}{n}.
	\end{align} 
This show that, under the assumptions taken, a scaling $n^{-1}$ is achievable and that the PACMAN bound is dominated by the number of parameters. It is also interesting to see the expected value behaviour: 
\begin{equation}
2n\mathop{\mathbb{E}}_{\substack{S\sim p(S)\\\mathbf{h}\sim{p(\mathbf{h}|S)}}}\left[\widetilde{\textup{CE}}_p(\mathbf{h})-\textup{CE}_{S}(\mathbf{h})\right]\leq\mathop{\mathbb{E}}_{S\sim p(S)}\left[2\log\frac{\displaystyle \sup_{\mathbf{h}\in{\mathcal{H}}}q(S|\mathbf{h})}{p(S)}\right]\overset{n\rightarrow\infty}{\longrightarrow}\nu,
\end{equation}
where also a $n^{-1}$ scaling is also achievable. In this case the bound $\mathop{\mathbb{E}}\left[\widetilde{\textup{CE}}_p(\mathbf{h})-\textup{CE}_{S}(\mathbf{h})\right]\lesssim\frac{\nu}{2n}$, depends almost trivially on the parametric model used through the total number of parameters $\nu$.

\end{example}
\section{Conclusion}\label{sec:conclution}

In this work we analyzed using a point-wise PAC analysis the term $\widetilde{\textup{CE}}_p(\mathbf{h})-\textup{CE}_{S}(\mathbf{h})$, so called PACMAN inequalities. This metric explicitly considers the mismatch between the accuracy, which is of interest during the testing phase, and the negative log-loss function, whose empirical counterpart is typically as cost function during training.  Several PACMAN bounds where obtained and some discussions about them at their links with previous results in the literature were provided. Interestingly, our analysis exploit the fact that $\widetilde{\textup{CE}}_p(\mathbf{h})-\textup{CE}_{S}(\mathbf{h})$ can be expressed as a likelihood ratio. From this observation and using a simple Chernoff bound we obtained several bounds that  depend on information theoretic metrics like mutual information and R\'enyi $\alpha$-divergence. For instance, Corollary \ref{cor:1} shows that the terms $H(\mathbf{h})$ and $\textup{CS}(S)$ defined in \eqref{eq:H-CSdefs}, control the  PAC point-wise bound for $\widetilde{\textup{CE}}_p(\mathbf{h})-\textup{CE}_{S}(\mathbf{h})$. Theo. \ref{thm:little} shows that $I(S;\mathbf{h})$ also controls the above mentioned metric. This is related with Theo. \ref{thm:bassily}, where differently that in the case considered in this work, the loss function is the same at the training and testing phases. Finally, in Theo. \ref{thm:ratio} a bound based on quantity $C_n(t)$ is obtained. As this quantity is not so well-known as other information theoretic quantities, some three simple examples were developed where this quantity can be easily bounded in terms on different types of structural information about the hypothesis space $\mathcal{H}$. It is important to emphasizes, that differently to some other similar approaches considered in the literature, all the bounds presented in this works presents a scaling with the training set of $n^{-1}$. 

\section*{Acknowledgment}
This project has received funding 
from CONICET under grant PIP11220150100578CO, and from University of Buenos Aires under grant UBACyT 20020170100470BA. 


\bibliographystyle{IMAIAI} 
\bibliography{mybibfile}
%

\appendix

\section{Appendix}

\subsection{About the hypothesis $R_p^\ast(h)\lesssim R_p(\mathbf{h})$}\label{sec:hardsoft}

In this appendix we show how the hypothesis of $R_p^\ast(h)\lesssim R_p(\mathbf{h})$ is totally reasonable in practice. Both expected hard-risk $R_p^\ast({h})$ and expected soft-risk $R_p(\mathbf{h})$ are minimized with a hard decision: the famous Bayesian classifier \cite[Ex. 2.11]{Duda}. This suggest that a classifier with low risk $R_p(\mathbf{h})$ will have $R_p^\ast(h)\approx R_p(\mathbf{h})$. 
It is well known that combining cross-entropy minimization and a soft-max final layer tends to concentrate the values of $\mathbf{h}(x)$ \cite[Sec. 6.2.2]{Goodfellow-et-al-2016} and this effect is exacerbated by large neural networks and current regularization methods, as batch normalization or weight decay \cite{guo17}. Therefore, when those algorithm are used the two risks should be closer to each other. In order to empirically test this, in Table \ref{tabla:FCFM} we show the different metric values (mean and standard deviation over $20$ independent simulations) for a $3$- layers fully connected deep neural network (two layers with the same hidden units with ReLU activation and a third softmax layer) trained with FashionMNIST dataset \cite{xiao2017/online} for different hidden units. Both risks are always very close and the hypothesis $R_p^\ast(h)\leq R_p(\mathbf{h})$ holds. Next, we present a theoretical result that gives some conditions for $R_p^\ast(h)\leq R_p(\mathbf{h})$ to hold exactly.

In this section we will study this phenomenon using the following result. 

\begin{theorem}\label{teo:hardsoft}
	Soft and hard risk are related by
\begin{equation}
\frac{1}{2}R_p^\ast(h)\leq R_p(\mathbf{h})\leq1-c_\mathbf{h}+R_p^\ast(h),
\end{equation}		
where $c_\mathbf{h}$ is the expected confidence of the decision:
	\begin{equation}
	c_\mathbf{h}=\mathop{\mathbb{E}}_{x\sim p(x)}\left[(\mathbf{h}(x))_{h(x)}\right].
	\end{equation}	
In addition, 
\begin{enumerate}
	\item If $R_p^\ast(h)\leq 1-c_\mathbf{h}$ then $R_p^\ast(h)\leq R_p(\mathbf{h})$.
	\item If $c_\mathbf{h}=1$ then $R_p^\ast(h)=R_p(\mathbf{h})$.
\end{enumerate}
\end{theorem}
\begin{proof}
We start the proof of $R_p^\ast(h)\leq 2R_p(\mathbf{h})$ using the following reverse Markov inequality:
\begin{align}
\mathop{\mathbb{E}}_{(x,y)\sim p(x,y)}\left[(\mathbf{h}(x))_{y}\right]&=\int_{0}^{\frac{1}{2}}\mathop{\mathbb{P}}_{(x,y)\sim p(x,y)}\left((\mathbf{h}(x))_{y}>t\right)dt+\int_{\frac{1}{2}}^{1}\mathop{\mathbb{P}}_{(x,y)\sim p(x,y)}\left((\mathbf{h}(x))_{y}>t\right)dt\\
&\leq\frac{1}{2}+\frac{1}{2}\mathop{\mathbb{P}}_{(x,y)\sim p(x,y)}\left((\mathbf{h}(x))_{y}>\frac{1}{2}\right).
\end{align}

It is easy to see that if there is a vector component greater than $\frac{1}{2}$, then it is the maximum: $\left\{(\mathbf{h}(x))_{y}>\frac{1}{2}\right\}\subset\{y=h(x)\}$. So $1-R_p(\mathbf{h})\leq\frac{1}{2}+\frac{1}{2}(1-R_p^\ast(h))$ and the first inequality is proved. 

For the the inequality $R_p(\mathbf{h})\leq1-c_\mathbf{h}+R_p^\ast(h)$ note that:
\begin{align}
(\mathbf{h(x)})_y&= (\mathbf{h(x)})_{h(x)}\mathds{1}\left\{y=h(x)\right\}+(\mathbf{h(x)})_y\mathds{1}\left\{y\neq h(x)\right\}\\
&=(\mathbf{h(x)})_{h(x)}-\left[(\mathbf{h(x)})_{h(x)}-(\mathbf{h(x)})_{y}\right]\mathds{1}\left\{y\neq h(x)\right\}\\
&\geq(\mathbf{h(x)})_{h(x)}-\mathds{1}\left\{y\neq h(x)\right\}.
\end{align}
Then taking expectation in both sides we get $1-R_p(\mathbf{h})\geq c_\mathbf{h}-R_p^\ast(h)$.

To conclude note that $c_\mathbf{h}\geq1-R_p(\mathbf{h})$ (the expectation of $(\mathbf{h}(x))_y$ is smaller than the expectation of its maximum over $\mathcal{Y}$). So if $R_p^\ast(h)\leq 1-c_\mathbf{h}$ then $R_p^\ast(h)\leq R_p(\mathbf{h})$. Similarly, note that $c_\mathbf{h}=1$ implies $(\mathbf{h}(x))_{y}=\mathds{1}\left\{y=h(x)\right\}$ almost surely, from which $R_p^\ast(h)=R_p(\mathbf{h})$ follows.
\end{proof}

\begin{remark}
Confidence is a magnitude well studied in calibration theory \cite{guo17}. It is said that an algorithm is calibrated if 
\begin{equation}
\mathop{\mathbb{P}}_{(x,y)\sim p(x,y)}\left(h(x)=y |(\mathbf{h}(x))_{h(x)}=p\right)=p
\end{equation}
almost surely, and this implies 
that $1-R_p^\ast(h)=c_\mathbf{h}$ 
. In this sense, the difference between these magnitudes can be bounded by the expected calibration error (ECE):
\begin{align}
|1-R_p^\ast(h)-c_\mathbf{h}|&=\left|\mathop{\mathbb{E}}_{(x,y)\sim p(x,y)}\left[\mathds{1}\left\{h(x)=y\right\}-(\mathbf{h}(x))_{h(x)}\right]\right|\\
&\leq\mathop{\mathbb{E}}_{\substack{p=(\mathbf{h}(x^\prime))_{h(x\prime)}\\x^\prime\sim p(x)}}\left[\left|\mathop{\mathbb{P}}_{(x,y)\sim p(x,y)}\left(h(x)=y |(\mathbf{h}(x))_{h(x)}=p\right)-p\right|\right]\\
&\triangleq\text{ECE}.
\end{align}
As it is well-known, neural networks typically present $\text{ECE}\geq 0$ \cite{guo17}.
\end{remark}

In general, we only have than $R_p^\ast(h)\leq 2R_p(\mathbf{h})$, which suggests that minimizing soft risk helps decreasing the hard one. Theo. {teo:hardsoft} shows that this relationship can be strengthened in situations such as calibrated algorithms $R_p^\ast(h)\leq R_p(\mathbf{h})$ and totally concentrated algorithms ($c_\mathbf{h}=1$) where $R_p^\ast(h)=R_p(\mathbf{h})$. 

\begin{table}[t]
	\centering	
	\begin{tabular}{||c c c c c||} 
		\hline
		Units & $\textup{CE}_{S}(\mathbf{h})$ & $R_p^\ast(h)$ & $R_p(\mathbf{h})$ & $\widetilde{\textup{CE}}_p(\mathbf{h})-\textup{CE}_{S}(\mathbf{h})$ \\ [0.5ex] 
		\hline\hline
4 & 0.509$\pm$0.16 & 0.204$\pm$0.063 & 0.277$\pm$0.07 & -0.179$\pm$0.034 \\
\hline
8 & 0.322$\pm$0.013 & 0.155$\pm$0.0033 & 0.201$\pm$0.0047 & -0.097$\pm$0.0089 \\
\hline
16 & 0.207$\pm$0.011 & 0.142$\pm$0.0029 & 0.168$\pm$0.003 & -0.0228$\pm$0.0097 \\
\hline
32 & 0.0911$\pm$0.012 & 0.136$\pm$0.0039 & 0.145$\pm$0.003 & 0.0657$\pm$0.012 \\
\hline
64 & 0.0425$\pm$0.011 & 0.123$\pm$0.0027 & 0.128$\pm$0.0026 & 0.0944$\pm$0.01 \\
\hline
128 & 0.0305$\pm$0.014 & 0.112$\pm$0.0052 & 0.116$\pm$0.005 & 0.0927$\pm$0.0089 \\
\hline
256 & 0.0343$\pm$0.012 & 0.107$\pm$0.003 & 0.111$\pm$0.0029 & 0.0836$\pm$0.0097 \\
\hline
512 & 0.0299$\pm$0.0083 & 0.104$\pm$0.0027 & 0.108$\pm$0.0027 & 0.0839$\pm$0.0065 \\
\hline
1024 & 0.0263$\pm$0.0081 & 0.102$\pm$0.0036 & 0.106$\pm$0.0036 & 0.0859$\pm$0.0054 \\
\hline
2048 & 0.0305$\pm$0.014 & 0.103$\pm$0.0038 & 0.107$\pm$0.0039 & 0.0823$\pm$0.011 \\
\hline
4096 & 0.0264$\pm$0.01 & 0.103$\pm$0.0031 & 0.106$\pm$0.0029 & 0.0858$\pm$0.0084 \\
\hline
\end{tabular}
\caption{Measures (mean and standard deviation) in a fully connected deep neural network trained with FashionMNIST dataset.}
\label{tabla:FCFM}
\end{table}

\subsection{About the behaviour of  $\widetilde{\textup{CE}}_p(\mathbf{h})-\textup{CE}_{S}(\mathbf{h})$}
\label{sec:signo}

In the case where the cross-entropy gap converges uniformly, i.e. $\displaystyle\sup_{\mathbf{h}\in{\mathcal{H}}}|\text{CE}_{S}(\mathbf{h})-{\text{CE}}_p(\mathbf{h})|\overset{n\rightarrow\infty}{\longrightarrow}0$ with probability one, using Jensen inequality, it can be shown the asymptotic value has negative behaviour $\widetilde{\textup{CE}}_p(\mathbf{h})-\textup{CE}_{S}(\mathbf{h})\leq {\textup{CE}}_p(\mathbf{h})-\textup{CE}_{S}(\mathbf{h})\overset{n\rightarrow\infty}{\longrightarrow}0$. However, in practice the uniform convergence property is not necessarily true for the cross-entropy loss and typical parametric model families.  It is observed that, in typical practical scenarios, $\widetilde{\textup{CE}}_p(\mathbf{h})-\textup{CE}_{S}(\mathbf{h})$ does present a positive bias which motivates the bounds studied in this paper. For example, in the last column of Table \ref{tabla:FCFM} we show the quantity $\widetilde{\textup{CE}}_p(\mathbf{h})-\textup{CE}_{S}(\mathbf{h})$ (mean and standard deviation over $20$ independent simulations) in a $3$- layers fully connected deep neural network for different hidden units. This effect is enhanced in even larger networks such as ResNet \cite{resnet}. It would seem that our bounds are more interesting in over-parameterized models. This is the case of large deep neural networks, in which the models are rich enough to memorize the training data \cite{DBLP:journals/corr/ZhangBHRV16}, i.e. $\text{CE}_{S}(\mathbf{h})\approx0$. However, even if this gap is not positive, it is lower bounded by the empirical result $\widetilde{\textup{CE}}_p(\mathbf{h})-\textup{CE}_{S}(\mathbf{h})\geq-\textup{CE}_{S}(\mathbf{h})$. In this sense, a bound on $|\widetilde{\textup{CE}}_p(\mathbf{h})-\textup{CE}_{S}(\mathbf{h})|$, can be easily related to the bounds obtained in this paper.
\begin{lemma}
Let $\epsilon(\delta,n,p(\mathbf{h}|S),p(S))$ a PACMAN bound, i.e.
\begin{equation}
\mathop{\mathbb{P}}_{\substack{S\sim p(S)\\\mathbf{h}\sim{p(\mathbf{h}|S)}}}\left(\widetilde{\textup{CE}}_p(\mathbf{h})\leq\textup{CE}_{S}(\mathbf{h})+\epsilon(\delta,n,p(\mathbf{h}|S),p(S))\right)\geq1-\delta.
\end{equation}
Then
\begin{equation}
\mathop{\mathbb{P}}_{\substack{S\sim p(S)\\\mathbf{h}\sim{p(\mathbf{h}|S)}}}\left(|\widetilde{\textup{CE}}_p(\mathbf{h})-\textup{CE}_{S}(\mathbf{h})|> \epsilon(\delta,n,p(\mathbf{h}|S),p(S))\right)\leq\delta+\mathop{\mathbb{P}}_{\substack{S\sim p(S)\\\mathbf{h}\sim{p(\mathbf{h}|S)}}}\bigg(\textup{CE}_{S}(\mathbf{h})> \epsilon(\delta,n,p(\mathbf{h}|S),p(S))\bigg).
\end{equation}
\end{lemma}
\begin{proof}
It is not hard to see that:
\small
\begin{align}
&\mathop{\mathbb{P}}_{\substack{S\sim p(S)\\\mathbf{h}\sim{p(\mathbf{h}|S)}}}\left(|\widetilde{\textup{CE}}_p(\mathbf{h})-\textup{CE}_{S}(\mathbf{h})|> \epsilon(\delta,n,p(\mathbf{h}|S),p(S))\right)\\
&=\mathop{\mathbb{P}}_{\substack{S\sim p(S)\\\mathbf{h}\sim{p(\mathbf{h}|S)}}}\left(\widetilde{\textup{CE}}_p(\mathbf{h})-\textup{CE}_{S}(\mathbf{h})> \epsilon(\delta,n,p(\mathbf{h}|S),p(S))\right)+\mathop{\mathbb{P}}_{\substack{S\sim p(S)\\\mathbf{h}\sim{p(\mathbf{h}|S)}}}\left(\widetilde{\textup{CE}}_p(\mathbf{h})-\textup{CE}_{S}(\mathbf{h})<- \epsilon(\delta,n,p(\mathbf{h}|S),p(S))\right)\nonumber\\
&\leq\delta+\mathop{\mathbb{P}}_{\substack{S\sim p(S)\\\mathbf{h}\sim{p(\mathbf{h}|S)}}}\left(\textup{CE}_{S}(\mathbf{h})> \epsilon(\delta,n,p(\mathbf{h}|S),p(S))\right),
\end{align}\normalsize
and where we used that $\widetilde{\textup{CE}}_p(\mathbf{h})\geq0$.
\end{proof}

{This results shows that in the important practical setting in which large deep models are used and where typically small empirical cross-entropy is achieved during training, the most important quantity to be analyzed is $\displaystyle \mathop{\mathbb{P}}_{\substack{S\sim p(S)\\\mathbf{h}\sim{p(\mathbf{h}|S)}}}\left(\widetilde{\textup{CE}}_p(\mathbf{h})\leq\textup{CE}_{S}(\mathbf{h})+\epsilon(\delta,n,p(\mathbf{h}|S),p(S))\right)$ which is the main object of study in this paper.}


\subsection{Proof of Lemma \ref{lem:paracomparar}}\label{app:paracomparar}

First note that:
\begin{align}
R_p(\mathbf{h})&=1-e^{-\widetilde{\textup{CE}}_p(\mathbf{h})}\\    
&=1-e^{-{\textup{CE}}_S(\mathbf{h})}+e^{-{\textup{CE}}_S(\mathbf{h})}\left(1-e^{-\left[\widetilde{\textup{CE}}_p(\mathbf{h})-\textup{CE}_{S}(\mathbf{h})\right]}\right)\\
&\leq\left(1-e^{-{\textup{CE}}_S(\mathbf{h})}\right)+e^{-{\textup{CE}}_S(\mathbf{h})}\left[\widetilde{\textup{CE}}_p(\mathbf{h})-\textup{CE}_{S}(\mathbf{h})\right],
\end{align}
where we used the inequality $1-e^{-x}\leq x$ for all $x\in\mathbb{R}$. Finally we can use \eqref{eq:aux_pacman} and write: 
\begin{align}
1-\delta&\leq \mathop{\mathbb{P}}_{\substack{S\sim p(S)\\\mathbf{h}\sim{p(\mathbf{h}|S)}}}\left(\widetilde{\textup{CE}}_p(\mathbf{h})-\textup{CE}_{S}(\mathbf{h})\leq\epsilon(\delta,n,p(\mathbf{h}|S),p(S))\right)\\
&\leq\mathop{\mathbb{P}}_{\substack{S\sim p(S)\\\mathbf{h}\sim{p(\mathbf{h}|S)}}}\left(\frac{R_p(\mathbf{h})-\left(1-e^{-\textup{CE}_{S}(\mathbf{h})}\right)}{e^{-\textup{CE}_{S}(\mathbf{h})}}\leq\epsilon(\delta,n,p(\mathbf{h}|S),p(S))\right).
\end{align}

\subsection{Proof of Theorem \ref{thm:bayes} and Corollary \ref{cor:1}}\label{app:bayes}

We define:
\begin{align}
X_1&=\frac{1}{n}\log\frac{q(S|\mathbf{h})}{p(S|\mathbf{h})},\qquad X_2=\frac{1}{n}\log\frac{p(\mathbf{h}|S)}{p(\mathbf{h})}. 
\end{align}
From \eqref{eq:2ratio} we can write $\widetilde{\text{CE}}_p(\mathbf{h})-\text{CE}_{S}(\mathbf{h})=X_1+X_2$. Let $\alpha\in(0,1)\cup(1,+\infty)$, $\beta\in(1,+\infty)$ and $\delta\in(0,1]$ and consider the following version of Chernoff inequality. For $t>0$: 
\begin{equation}
\mathop{\mathbb{P}}_{X\sim p}\left(X>\epsilon\right)\leq e^{-t\epsilon}\mathop{\mathbb{E}}_{X\sim p}\left[e^{tX}\right]\quad\Rightarrow\quad \mathop{\mathbb{P}}_{X\sim p}\left(X>\frac{1}{t}\log\frac{\displaystyle\mathop{\mathbb{E}}_{X\sim p}\left[e^{tX}\right]}{\delta_0}\right)\leq \delta_0,
\end{equation}
where $\delta_0=e^{-t\epsilon}\mathop{\mathbb{E}}_{X\sim p}\left[e^{tX}\right]$. We apply previous Chernoff inequality to obtain:
\begin{align}
\mathop{\mathbb{P}}_{S\sim p(S|\mathbf{h})}&\left(X_1>\frac{1}{\alpha n}\log\left[\frac{2}{\delta}\mathop{\mathbb{E}}_{S\sim p(S|\mathbf{h})}\left(\frac{q(S|\mathbf{h})}{p(S|\mathbf{h})}\right)^{\alpha}\right]\right)\leq\frac{\delta}{2}\quad\textup{(for each }\mathbf{h})\\
\mathop{\mathbb{P}}_{\substack{\mathbf{h}\sim p(\mathbf{h})\\S\sim p(S|\mathbf{h})}}&\left(X_1>\frac{1}{\alpha n}\log\left[\frac{2}{\delta}\mathop{\mathbb{E}}_{S\sim p(S|\mathbf{h})}\left(\frac{q(S|\mathbf{h})}{p(S|\mathbf{h})}\right)^{\alpha}\right]\right)\leq\frac{\delta}{2}\\
\mathop{\mathbb{P}}_{\mathbf{h}\sim{p(\mathbf{h}|S)}}&\left(X_2>\frac{1}{(\beta-1)n}\log\left[\frac{2}{\delta}\mathop{\mathbb{E}}_{\mathbf{h}\sim p(\mathbf{h}|S)}\left(\frac{p(\mathbf{h}|S)}{p(\mathbf{h})}\right)^{\beta-1}\right]\right)\leq\frac{\delta}{2}\quad\textup{(for each }S)\\
\mathop{\mathbb{P}}_{\substack{S\sim p(S)\\\mathbf{h}\sim{p(\mathbf{h}|S)}}}&\left(X_2>\frac{1}{(\beta-1)n}\log\left[\frac{2}{\delta}\mathop{\mathbb{E}}_{\mathbf{h}\sim p(\mathbf{h}|S)}\left(\frac{p(\mathbf{h}|S)}{p(\mathbf{h})}\right)^{\beta-1}\right]\right)\leq\frac{\delta}{2},
\end{align}
and where $t=\alpha n$, $\delta_0=\frac{\delta}{2}$ for $X_1$ and $t=(\beta-1) n$, $\delta_0=\frac{\delta}{2}$ for $X_2$. Using the \emph{union bound} we can write:
\begin{align}
\mathbb{P}(X_1+X_2>\epsilon_1+\epsilon_2)&\leq\mathbb{P}(\{X_1>\epsilon_1\}\cup\{X_2>\epsilon_2\})\\
&\leq\mathbb{P}(X_1>\epsilon_1)+\mathbb{P}(X_2>\epsilon_2)
\label{eq:PX}
\end{align}
for all $\epsilon_1,\epsilon_2$. In this context, we can apply the aforementioned Chernoff inequalities to each term in (\ref{eq:PX}) to obtain, with probability at least $1-\delta$:
\small
\begin{align}\label{eq:2chernoff}
\widetilde{\text{CE}}_p(\mathbf{h})&\leq\text{CE}_{S}(\mathbf{h})+\frac{1}{\alpha n}\log\left[\frac{2}{\delta}\mathop{\mathbb{E}}_{S\sim p(S|\mathbf{h})}\left(\frac{q(S|\mathbf{h})}{p(S|\mathbf{h})}\right)^{\alpha}\right]+\frac{1}{(\beta-1)n}\log\left[\frac{2}{\delta}\mathop{\mathbb{E}}_{\mathbf{h}\sim p(\mathbf{h}|S)}\left(\frac{p(\mathbf{h}|S)}{p(\mathbf{h})}\right)^{\beta-1}\right]\\
&=\text{CE}_{S}(\mathbf{h})+\frac{1}{\alpha n}\log\left(\frac{2}{\delta}e^{(\alpha-1)D_\alpha(q(S|\mathbf{h})\|p(S|\mathbf{h}))}\right)+\frac{1}{(\beta-1)n}\log\left(\frac{2}{\delta}e^{(\beta-1)D_\beta(p(\mathbf{h}|S)\|p(\mathbf{h}))}\right)\\
&=\text{CE}_{S}(\mathbf{h})+\frac{\alpha-1}{\alpha n}D_\alpha(q(S|\mathbf{h})\|p(S|\mathbf{h}))+\frac{1}{n}D_\beta(p(\mathbf{h}|S)\|p(\mathbf{h}))+\left(\frac{1}{\alpha n}+\frac{1}{(\beta-1)n}\right)\log\frac{2}{\delta},
\end{align}\normalsize
where we have used the identity  $\mathop{\mathbb{E}}_{\mathbf{h}\sim p(\mathbf{h}|S)}\left(\frac{p(\mathbf{h}|S)}{p(\mathbf{h})}\right)^{\beta-1}=\mathop{\mathbb{E}}_{\mathbf{h}\sim p(\mathbf{h})}\left(\frac{p(\mathbf{h}|S)}{p(\mathbf{h})}\right)^{\beta}$. This concludes the proof of Theo. \ref{thm:bayes} 

In order to prove Corollary \ref{cor:1}, we study two particular cases of R\'enyi $\alpha$-divergences ($\alpha=1/2$ and $\alpha=2$), which are related to the Hellinger distance and square chi divergence respectively \cite{gibbsS02}:
	\begin{align}
	\textup{Hel}^2(q(S|\mathbf{h}),p(S|\mathbf{h}))&=2\left(1-e^{-\frac{1}{2}D_{1/2}(q(S|\mathbf{h})\|p(S|\mathbf{h})}\right)=2\left(1-\mathop{\mathbb{E}}_{S\sim p(S|\mathbf{h})}\left[\left(\frac{q(S|\mathbf{h})}{p(S|\mathbf{h})}\right)^{1/2}\right]\right),\\
	\chi^2(p(\mathbf{h}|S)\|p(\mathbf{h}))&=e^{D_2(p(\mathbf{h}|S)\|p(\mathbf{h}))}-1=\mathop{\mathbb{E}}_{\mathbf{h}\sim p(\mathbf{h}|S)}\left[\frac{p(\mathbf{h}|S)}{p(\mathbf{h})}\right]-1.
	\end{align}
Choosing $\alpha=\frac{1}{2}$ and $\beta=2$, Corollary \ref{cor:1} is obtained.

\end{document}